\documentclass[a4paper,11pt]{article}

\usepackage[font=concrete]{chihao}
\usepackage{graphicx}
\usepackage[margin=1in]{geometry}
\usepackage{multirow}
\usepackage{array}

\newcommand{\Ber}[1]{\!{Ber}\tp{#1}}
\newcommand{\arm}{\!{arm}\xspace}

\begin{document}
\title{Tight Gap-Dependent Memory-Regret Trade-Off for Single-Pass Streaming Stochastic Multi-Armed Bandits}
%
%
\author{Zichun Ye \\ Shanghai Jiao Tong University \\ \textsf{alchemist@sjtu.edu.cn}
\and Chihao Zhang  \\ Shanghai Jiao Tong University \\ \textsf{chihao@sjtu.edu.cn}
\and Jiahao Zhao \\ The University of Hong Kong \\ \textsf{zjiahao@connect.hku.hk}
}


%
\maketitle              
\begin{abstract}
We study the problem of minimizing gap-dependent regret for single-pass streaming stochastic multi-armed bandits (MAB). In this problem, the $n$ arms are present in a stream, and at most $m<n$ arms and their statistics can be stored in the memory. We establish tight \emph{non-asymptotic} regret bounds regarding all relevant parameters, including the number of arms $n$, the memory size $m$, the number of rounds $T$ and $(\Delta_i)_{i\in [n]}$ where $\Delta_i$ is the reward mean gap between the best arm and the $i$-th arm. These gaps are \emph{not} known in advance by the player. Specifically, for any constant $\alpha \ge 1$, we present two algorithms:  one applicable for $m\ge \frac{2}{3}n$ with regret at most $O_\alpha\Big(\frac{(n-m)T^{\frac{1}{\alpha + 1}}}{n^{1 + {\frac{1}{\alpha + 1}}}}\displaystyle\sum_{i:\Delta_i > 0}\Delta_i^{1 - 2\alpha}\Big)$\footnote{In this paper, the notations $O_\alpha, \Omega_\alpha, \Theta_\alpha$ subsume a multiplicative factor depending only on $\alpha$. This is fine since we usually take $\alpha$ to be a constant.} and another applicable for $m<\frac{2}{3}n$ with regret at most $O_\alpha\Big(\frac{T^{\frac{1}{\alpha+1}}}{m^{\frac{1}{\alpha+1}}}\displaystyle\sum_{i:\Delta_i > 0}\Delta_i^{1 - 2\alpha}\Big)$. We also prove matching lower bounds for both cases by showing that for any constant $\alpha\ge 1$ and any $m\leq k < n$, there exists a set of hard instances on which the regret of any algorithm is $\Omega_\alpha\Big(\frac{(k-m+1) T^{\frac{1}{\alpha+1}}}{k^{1 + \frac{1}{\alpha+1}}} \sum_{i:\Delta_i > 0}\Delta_i^{1-2\alpha}\Big)$. This is the first tight gap-dependent regret bound for streaming MAB. Prior to our work, an $O\Big(\sum_{i\colon\Delta>0} \frac{\sqrt{T}\log T}{\Delta_i}\Big)$ upper bound for the special case of $\alpha=1$ and $m=O(1)$ was established by  Agarwal, Khanna and Patil (COLT'22). In contrast, our results provide the correct order of regret as $\Theta\Big(\frac{1}{\sqrt{m}}\sum_{i\colon\Delta>0}\frac{\sqrt{T}}{\Delta_i}\Big)$.

\end{abstract}

\newpage

\setcounter{tocdepth}{1}
\tableofcontents

\section{Introduction}

The stochastic multi-armed bandits (MAB) is a popular $T$-round game that has been widely studied in online learning. In the game, one player faces $n$ arms. In each round $t \in [T]$, the player chooses an arm $A_t$ among the $n$ arms and gets a reward drawn from a predetermined reward distribution with mean $\mu_{A_t} \in [0,1]$. The arm with the largest reward mean $\mu_*$ is called the best arm. The total expected regret is defined as
$
    \E{R(T)} = \E{\sum_{t=1}^T \mu_* - \mu_{A_t}}, 
$
where the expectation is over the randomness from the player's strategy, and the aim is to minimize $\E{R(T)}$.

The classic MAB problem defined above has been thoroughly studied. It is known that the expected minimax regret, namely the regret of the best algorithm against the worst input, is $\Theta\tp{\sqrt{n T}}$ via the \emph{Upper Confidence Bounds} (UCB) algorithm and its variants (see e.g.~\cite{AB09,Aue02,GC11}). As for the \emph{gap-dependent} regret, the mean gap $\Delta_i\defeq \mu_*-\mu_i$ for each $i\in [n]$ is involved in such bound. The UCB algorithm provides an regret upper bound $O\tp{\sum_{i\in [n]\colon \Delta_i>0}\frac{\log T}{\Delta_i}}$ (see e.g.~\cite{GC11,KCG12,LS20}).

A recent line of research modeled the MAB problem in the streaming setting to incorporate the situation where the number of arms is huge and cannot be stored in the memory at the same time. In this model, the $n$ arm arrives one by one in a stream, and only $m<n$ arms and their statistics can be stored at the same time. When an arm is read into memory and stored, it can be pulled and its corresponding statistics can be stored. Once an arm is discarded from the memory, all of its information will be forgotten and can never be pulled again. The minimax regret of the problem has recently been settled in~\cite{HYZ25,Wan23}, which is $\Theta\tp{\frac{n-m}{n^{\frac{2}{3}}}\cdot T^{\frac{2}{3}}}$ for any $2\le m\le n-1$. 

However, the gap-dependent regret in this setting has not been fully explored. In fact, the minimax regret bound is derived by considering the worst cases over all choices of possible gaps, and might be much worse compared to the case where the actual gap values are explicitly taken into account. To the best of our knowledge, the only known upper bound is $O\tp{\sum_{i: \Delta_i > 0}\frac{\sqrt{T}\log T}{\Delta_i}}$, proved in~\cite{AKP22}, which holds only for $m=O(1)$. This upper bound already suggests that the gap-dependent regret bound can be superior to the minimax regret bound ($\wt O(\sqrt{T})$ v.s. $O\tp{T^{\frac{2}{3}}}$). However, as mentioned before, it is well known that when no memory constraint is considered, the dependency on $T$ in the gap-dependent bound can be as low as $\log T$.  Therefore, it is natural to ask what is the correct regret bound in the memory-constrained setting, and particularly how the memory affects the bound. 

 On the other hand, since the mean gap  $\Delta_i$'s are part of the input instance and might depend on $T$, there might be a trade-off between the dependency on $\frac{1}{\Delta_i}$ and $T$ in the regret bound. To capture this trade-off, we introduce a new parameter $\alpha\ge 1$ and aim at establishing the regret bounds of the form $f(\alpha, m,n,T)\cdot \sum_{i\colon \Delta_i>0} \Delta_i^{1-2\alpha}$ for some function $f$. Therefore previous gap-dependent bounds, either with or without memory constraint, correspond to the case $\alpha=1$. 



\subsection{Our results}
In this work, we design new algorithms and prove matching regret lower bounds for the problem, confirming that the gap-dependent regret is $\Theta_\alpha\tp{\frac{(n-m)T^{\frac{1}{\alpha + 1}}}{n^{1 + {\frac{1}{\alpha + 1}}}}\displaystyle\sum_{i:\Delta_i > 0}\Delta_i^{1 - 2\alpha}}$ when $m\ge \frac{2}{3}n$ and $\Theta_\alpha\tp{\frac{T^{\frac{1}{\alpha + 1}}}{m^{{\frac{1}{\alpha + 1}}}}\displaystyle\sum_{i:\Delta_i > 0}\Delta_i^{1 - 2\alpha}}$ when $2\leq m < \frac{2}{3}n$ for any constant $\alpha\ge 1$. Our results are summarized in \Cref{tab:results}.

\begin{table}[H] 
    \caption{Summary of results for streaming MAB}
    \label{tab:results}
    \renewcommand{\arraystretch}{2}
    \begin{center}
        \begin{tabular}{m{0.2\textwidth}<{\centering} m{0.3\textwidth}<{\centering} m{0.3\textwidth}<{\centering}}
            ~ & Regret Bounds & Memory  \\
           \hline
           \multirow{2}*{\cite{Wan23}} & $O\tp{n^{\frac{1}{3}} T^{\frac{2}{3}}}$ & $m=\Theta(\log^* n)$ \\
            \cline{2-3}
            ~ & $\Omega\tp{n^{\frac{1}{3}} T^{\frac{2}{3}}}$ & $m\leq \frac{n}{20}$ \\
            \hline
            \cite{HYZ25} & $\Theta\tp{\frac{n-m}{n^{\frac{2}{3}}}\cdot T^{\frac{2}{3}}}$ & $2\leq m < n$\\
            \hline
            \cite{AKP22} & $O\tp{\sum_{i: \Delta_i > 0}\frac{\sqrt{T}\log T}{\Delta_i}}$ & $m = O(1)$\\
            \hline
            \multirow{2}*{\scriptsize [This work]} & $\alpha = 1,\;\Theta\tp{\frac{(n-m)\sqrt{T}}{n^{\frac{3}{2}}}\displaystyle\sum_{i:\Delta_i > 0}\frac{1}{\Delta_i}}$ & $m\ge \frac{2}{3}n$ \\
            \cline{2-3}
            ~ & $\alpha = 1,\;\Theta\tp{\frac{\sqrt{T}}{m^{\frac{1}{2}}}\displaystyle\sum_{i:\Delta_i > 0}\frac{1}{\Delta_i}}$ & $2\leq m < \frac{2}{3}n$ \\
            \hline
            \multirow{2}*{\scriptsize [This work]} & $\forall \alpha\ge 1,\;\Theta_\alpha\tp{\frac{(n-m)T^{\frac{1}{\alpha + 1}}}{n^{1 + {\frac{1}{\alpha + 1}}}}\displaystyle\sum_{i:\Delta_i > 0}\Delta_i^{1 - 2\alpha}}$ & $m\ge \frac{2}{3}n$ \\
            \cline{2-3}
            ~ & $\forall \alpha\ge 1,\;\Theta_\alpha\tp{\frac{T^{\frac{1}{\alpha + 1}}}{m^{{\frac{1}{\alpha + 1}}}}\displaystyle\sum_{i:\Delta_i > 0}\Delta_i^{1 - 2\alpha}}$ & $2\leq m < \frac{2}{3}n$ \\
            \hline
        \end{tabular}
    \end{center}
\end{table}


Similar to the minimax regret case, the algorithm for the large memory case ($m\ge \frac{2}{3}n$) differs from that of the small memory case ($m<\frac{2}{3}n$). The reason is that the player continually faces the task of determining which arm to discard from the memory during the game. The task is called the \emph{best arm retention} (BAR) problem and has been recently thoroughly studied~\cite{HYZ25,CHZ24}. It is known that the complexity of the problem is the same as that of the best arm identification (BAI) problem when the memory is small while a more efficient algorithm exists when the memory is large. Therefore, algorithms tailored for both small and large memory are necessary.  

Our bounds suggest many interesting behaviors of the model. Taking $\alpha=1$, our results reveal that with $m=n-1$, namely with only one unit less memory, the dependency on $T$ significantly increases from $\log T$ to $\sqrt{T}$. 
Another interesting phenomenon shown by our results is that the regret is not ``smooth'' in $m$, as shown in \Cref{fig:results} when $\alpha=1$\footnote{Since we only provide non-asymptotic bounds, the curve in \Cref{fig:results} demonstrates the regret bound qualitatively. In fact, the threshold $m=\frac{2}{3}n$ can be replaced by any $m=c\cdot n$ for constant $c\in [\frac{2}{3},1)$.}. This is in sharp contrast with the minimax regret $\Theta\tp{\frac{n-m}{n^{\frac{2}{3}}}\cdot T^{\frac{2}{3}}}$ which is smooth with respect to $m$ for all $2\le m<n$. Such non-smoothness appears in our lower bound proof in the following manner. We essentially show that give any $\alpha\ge 1$ and any $m\leq k < n$, there exists a set of hard instances on which any algorithm incurs regret $\Omega\tp{16^{-\alpha}\cdot\frac{(k-m+1) T^{\frac{1}{\alpha+1}}}{k^{1 + \frac{1}{\alpha+1}}} \sum_{i:\Delta_i > 0}\Delta_i^{1-2\alpha}}$. Therefore, the best lower bound is obtained by optimizing $k$. For $m < \frac{2}{3}n$, the best choice is $k=\frac{3}{2}m$ while for $m\ge \frac{2}{3}n$, the best choice is $k=n-1$. 

It is helpful to compare our results with the previous ones for $\alpha=1$. In this case, our bound is $\Theta\tp{\frac{n-m}{n^{\frac{3}{2}}}\sum_{i\colon\Delta_i>0}\frac{\sqrt{T}}{\Delta_i}}$ when $m\ge \frac{2}{3}n$ and $\Theta\tp{\frac{1}{\sqrt{m}}\sum_{i\colon\Delta_i>0}\frac{\sqrt{T}}{\Delta_i}}$ when $2\leq m < \frac{2}{3}n$. It improves the previous best bound in~\cite{AKP22}, which only holds for $m=O(1)$, and also fills the blank space in the large memory setting. 

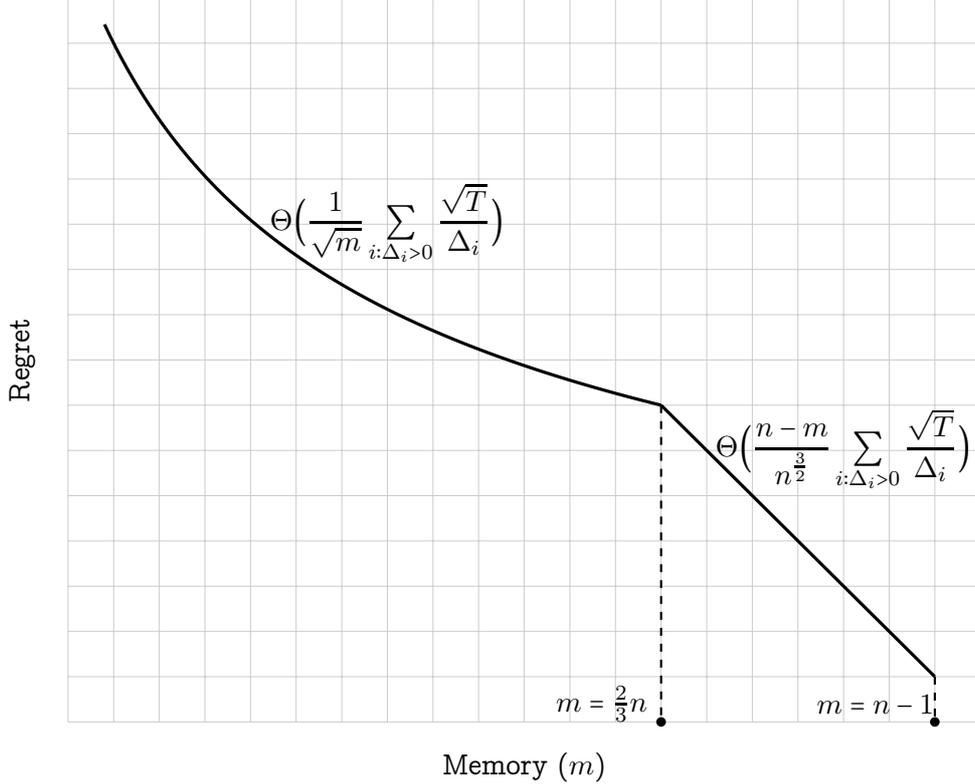
\begin{figure}
    \centering  
    \begin{tikzpicture}[scale=0.6]
    
        \draw[very thin,color=gray!40] (3,1) grid (23,17);
        
        \draw[very thick, domain=3.8:16,samples=100,black] plot (\x,{32/sqrt(\x)});
        
        \draw[very thick, domain=16:22,samples=100,black] plot (\x,{-\x + 24});
    
        \node at (10,12) {$\Theta\Big(\displaystyle\frac{1}{\sqrt{m}}\sum_{i: \Delta_i > 0} \frac{\sqrt{T}}{\Delta_i}\Big)$};
    
        \node at (20,7) {$\Theta\Big(\displaystyle\frac{n - m}{n^{\frac{3}{2}}}\sum_{i: \Delta_i > 0} \frac{\sqrt{T}}{\Delta_i}\Big)$};
    
        \draw[thick, dashed, color=black] (16,8) -- (16,1);
        \draw[thick, dashed, color=black] (22,2) -- (22,1);
    
        \node at (14.7, 1.4) {$m = \frac{2}{3}n$};
        \node at (20.7, 1.4) {$m = n - 1$};
        \fill[black] (16,1) circle (3pt);
        \fill[black] (22,1) circle (3pt);
        \node at (13, 0) {Memory ($m$)};
        \node[rotate= 90] at (2, 9) {Regret};
    \end{tikzpicture}
    \caption{Regret with respect to the memory size $m$.}
    \label{fig:results}
\end{figure}

\subsection{Related work} 
The MAB problem was first introduced in~\cite{Rob52}. The work \cite{AB09} proved an optimal regret bound of $\Theta(\sqrt{nT})$ for both stochastic and adversarial cases. Then \cite{LSPY18} first took the streaming MAB problem into consideration and obtained an instance-dependent upper bound using $O(\log T)$ passes and $O(1)$ memory. The work of~\cite{CK20} gave a generalized upper bound of $O\tp{\frac{n^{\frac{3}{2}}}{m}\sqrt{T\log \frac{T}{nm}}}$ for $2\leq m<n$ in $O(\log T)$ passes. And \cite{AW24,KW25} further explores the sample-pass trade-offs. Subsequently, the works of \cite{MPK21} and \cite{Wan23} studied the single-pass scenario and \cite{Wan23} gave tight regret bounds of $\Theta\tp{n^{\frac{1}{3}}T^{\frac{2}{3}}}$ when $\log^* n\leq m \leq \frac{n}{20}$. The work of \cite{AKP22} provided the first minimax regret lower bound with regard to the  number of passes $P$. They also gave an instance-dependent upper bound of $O\tp{\sum_{i: \Delta_i > 0}\frac{\sqrt{T}\log T}{\Delta_i}}$ when $m = O(1)$ and $P=1$. Very recently,~\cite{HYZ25} studied the minimax regret bound in the multi-pass setting and obtained tight bounds in terms of $m,n,T$ and $P$. In particular, they obtained a tight bound of $\Theta\tp{\frac{n-m}{n^{\frac{2}{3}}}\cdot T^{\frac{2}{3}}}$ in a single pass.

\section{Preliminaries} \label{sec:prelim}
\paragraph{Multi-armed bandit (MAB)}

We defined the problem of MAB at the beginning of the introduction. Here we introduce some further notations and define the game in detail. We use a mean vector $\nu = \tp{\mu_1, \mu_2, \cdots, \mu_n}$ to denote an instance of MAB in which the $i$-th arm has a Bernoulli reward  $\Ber{\mu_i}$. We use $i^*=\arg\max_{i\in [n]} \mu_i$ to denote the index of the arm with maximum mean reward. We assume the choice of $i^*$ is unique. We will sometimes call the $i$-th arm $\arm_i$ and write $\arm_{i^*}$ as $\arm_*$. We will also write $\mu_*$ instead of $\mu_{i^*}$ for convenience.

For every $i\in [n]$, we use a random variable $T_i \defeq \sum_{t=1}^T  \1{A_t = i}$ to denote the number of times the $i$-th arm has been pulled during the game. We use $\Pr[\nu]{\cdot}$ and $\E[\nu]{\cdot}$ to denote the probability and expectation of the algorithm running on instance $\nu$.

In the MAB game, the player is given a set of $n$ arms, denoted as $[n]$. Each arm has a reward mean $\mu_i$ from a fixed distribution $\+D_i$. A $T$-round decision game starts as follows: in each round $t\in[T]$, the player first selects an arm $A_t\in[n]$ to pull based on the information observed in previous rounds; then the player observes and gains reward of $r_t(A_t)\sim \+D_{A_t}$. The player's objective is to minimize the difference between the cumulative reward of the best arm and the player's own cumulative reward. That is, the player aims to design an algorithm $\+A$ to minimize the expected regret 
$
\E{R(T)} = \E{\sum_{t=1}^T r_t(i^*) - r_t(A_t)} = \E{\sum_{t=1}^T \mu_* - \mu_{A_t}}.
$

\paragraph{Streaming stochastic MAB.}
The streaming MAB model was first formalized in~\cite{LSPY18}. In this model, the $n$ arm arrives sequentially in the stream and the number of arms that can be stored at the same time is substantially smaller than $n$. A single pass means every arm in the stream only comes once, that is, if an arm is not stored yet or has been discarded from the memory, it cannot be retrieved later. We consider the worst-case order of the stream. 

In each round $t\in [T]$, the player acts in two stages, which include manipulating arm storage (discard arms in the memory, read new arms in the stream) and pulling arm (choose an arm $A_t$ in the memory to pull, observe and gain rewards) respectively. We emphasize that in one round, the player can discard and read any number of arms (including zero) but can only pull exactly once.

\paragraph{Best arm retention.}
To obtain the lower bound, we observe that it is important for an algorithm to pull an arm enough times  before discarding it. Otherwise, it can probably discard all good arms. The work of~\cite{CHZ24} modeled and called it the \textit{best arm retention} (BAR) problem. The BAR problem is to retain $m$ arms out of all $n$ arms after $T$ rounds and make sure that the best arm is not discarded with high probability. In other words, the player has to discard $n-m$ relatively bad arms during the game. A $\tp{\eps,\delta}$-PAC algorithm for BAR problem satisfies that for any fixed parameter $\eps,\delta \in (0,1)$, it retains an $\eps$-best arm in the $m$ arms with probability at least $1-\delta$. 

\paragraph{Gap-dependent bound.}

This paper focuses on gap-dependent regret bounds. Define the gap vector of a certain instance $\nu$ as $S(\nu) = (\Delta_i(\nu))_{i\in [n]}$ where $\Delta_i(\nu)$ is the mean gap between the best arm and the $i$-th arm in the instance $\nu$. For an instance $\nu$ of MAB, the gap-dependent regret ideally involves parameters $m,n,T$ and the gap vector $S(\nu)$. In this work, we define the \emph{lower bound} for the gap-dependent regret in the following sense.

\begin{definition}\label{def:lb}
    Let $m,n$ be fixed. Denote $\Pi$ as the set of all possible algorithms with memory size $m$, and $\@I$ as the set of all possible instances with $n$ arms. We say $R(m,n,T,S)$ is a non-asymptotic gap-dependent regret lower bound, if and only if for any algorithm $\+A \in \Pi$, there exists $\nu \in \@I$ such that
    $
        R^{\+A,\nu}(T) \ge R(m,n,T,S(\nu)),
    $
    for sufficiently large $T$, where $R^{\+A,\nu}(T)$ represents the regret incurred by running $\+A$ on $\nu$ after $T$ rounds.
\end{definition}
The \emph{non-asymptotic bounud} means that $T$ is a part of the input and therefore $\Delta_i$ might depend on $T$. Compared to the minimax lower bound, $R(m,n,T,S)$ is a functional depending on the gap-vector $S(\nu)=(\Delta_i(\nu))_{i\in [n]}$ where each $\Delta_i(\cdot)$ is a function on the instance. Nevertheless, we will use $R(T)$ as a shorthand for $R(m,n,T,S)$ when no ambiguity arises.

We will use the UCB algorithm as a black box and place the details of it in \Cref{sec:ucb-app}. We will use the following important property of the UCB algorithm in our analysis.
\begin{lemma}[\cite{Aue02}] \label{lem:ucb-sc}
    Giving $n$ arms which have Bernoulli rewards in the memory and running UCB on them for $T$ rounds, then for any $\arm_i$ with $\Delta_i > 0$:
   $
    \E{T_i} \leq \frac{8\log T}{\Delta_i^2}.
    $
\end{lemma}

We will also use the following technical lemma, which is a simple consequence of the Hoeffding's inequality. Its proof is provided in \Cref{sec:ub-prelim}.

\begin{lemma}\label{lem:worse-win}
    Let $\arm_1$ and $\arm_2$ be two different arms with reward means of $\mu$ and $\mu+\Delta, \Delta>0$ respectively. Suppose we sample each arm $L$ times and obtain empirical mean of $\wh{\mu}_1$ and $\wh\mu_2$ respectively, then 
    $
        \Pr{\wh\mu_1 \ge \wh{\mu}_2} \leq e^{-\frac{L\Delta^2}{2}}.
    $
\end{lemma}

\section{Gap-Dependent Regret Upper Bounds} \label{sec:ub}
In this section, we will propose two algorithms, which apply to large $m$ and small $m$ respectively. The philosophy behind the two algorithms is the same: one tries to retain the good arm in the memory and therefore solves a BAR problem in the streaming setting. However, when the memory is large, this can be done more efficiently.



\subsection{Large memory case (\texorpdfstring{$\frac{2}{3}n \leq m \leq n-1$}{})} \label{sec:l-mem}

When $m\ge \frac{2}{3}n$, we can simply read the first $m$ arms into the memory and have enough room to replace parts of them by rest arms in one batch. We simply pick $n-m$ pairs out of the $m$ arms, compare them pairwise, and replace the $n-m$ worst arms with the remaining $n-m$ fresh arms in the stream. After the manipulation of the memory, we apply a standard UCB algorithm for arms in the memory in the remaining rounds. 


\begin{algorithm}[H]
\caption{Single-pass algorithm for MAB when $\frac{2}{3}n \leq m \leq n-1$}
\label{alg:l-mem}
\Input{Time horizon $T$, memory size $m$, number of arms $n$ and a constant $\alpha \ge 1$.}
    \begin{algorithmic}[1]
        \State Let $L \leftarrow \tp{\frac{2\alpha}{e}}^{\frac{\alpha}{\alpha+1}}\cdot \tp{\frac{T}{n}}^\frac{1}{\alpha + 1}$, $c \leftarrow n-m$;
        \State Read in the first $m$ arms;
        \State Choose $2c$ arms from the memory u.a.r and denote them by $S=\set{s_1,s_2,\dots,s_{2c}}$; \label{line:findS}
        \For{$i = 1, 2, \dots, c$}
            \State Pull $s_{2i - 1}, s_{2i}$ each $L$ times, calculate their empirical means respectively;
            \State Discard the one with less empirical mean; \label{line:beat}
        \EndFor
        \State Read in the remaining $c$ arms in the stream, denote all the $m$ arms in memory as $M$;
        \State Run UCB on $M$ until the game ends;  \label{line:ucb-l-mem} \Comment{The Exploitation Phase}
    \end{algorithmic}
\end{algorithm}

The strategy for the analysis of the algorithm is as follows. If the best arm $\arm_*$ does not show up in the first $m$ arms, then it must belong to the last $M$, and thus the UCB algorithm will take care of everything. Otherwise, $\arm_*$ has a probability of $\frac{2c}{m}$ to be chosen into $S$. Then we carefully analyze its probability of being beaten by another sub-optimal arm in $L$ rounds and deduce the bound for regret incurred by this bad event. We obtain the following theorem.

\begin{theorem} \label{thm:l-ub}
    Given any $\alpha \ge 1$ and any input instance, assuming $T$ is sufficiently large, \Cref{alg:l-mem} uses the memory of $m$ arms with expected regret
    \[
    \E{R(T)} = O\tp{\tp{\frac{2\alpha}{e}}^{\frac{\alpha}{\alpha+1}} \frac{\tp{n-m}T^{\frac{1}{\alpha + 1}}}{n^{1 + {\frac{1}{\alpha + 1}}}}\displaystyle\sum_{i:\Delta_i > 0}\Delta_i^{1 - 2\alpha}}.
    \]
\end{theorem}

\begin{proof}
We first define some events in the probability space induced by running our algorithm on a fixed instance. For every $i\in [n]\cup\{*\}$,
\begin{itemize}
    \item $\+S_i:$ the $\arm_i$ is in the set $S$ found in \Cref{line:findS} of \Cref{alg:l-mem}.
    \item $\+C_i:$ the $\arm_i$ and $\arm_*$ are placed in the same group. 
    \item $\+B_i:$ the $\arm_i$ \emph{beats} $\arm_*$ (the empirical mean of $\arm_i$ is larger than $\arm_*$ in \Cref{line:beat} of \Cref{alg:l-mem}).
    \item $\+M_i:$ the $\arm_i$ is in the final collection of arms $M$.
\end{itemize}
Note that when $i=*$, we mean the best arm.
We use $L_1$ to denote the number of rounds before the execution of \Cref{line:ucb-l-mem} in \Cref{alg:l-mem} and $L_2$ be the remaining rounds consumed in \Cref{line:ucb-l-mem}. For each $i\in [n]$, we denote by $L_{1,i}$ and $L_{2,i}$ the number of rounds playing the $\arm_i$ during the first $L_1$ and the last $L_2$ rounds respectively. Clearly $L_1+L_2 = T$ and $\sum_{i\in [n]} L_{c,i}=L_c$ for $c=1,2$. We also use $R_1$ and $R_2$ to denote the regret incurred in the first $L_1$ rounds and the last $L_2$ rounds respectively. Then $R(T) = R_1+R_2$.


According to the position of $\arm_*$ in the stream, we classify the proof into two cases.

\paragraph{Case 1:} The $\arm_*$ is in the last $c=n-m$ arms. Therefore $\Pr{\+M_*} = 1$, which means that we can bound the regret in the exploitation phase using \Cref{lem:ucb-sc}. That is,
 
\begin{align*}
    \E{R(T)} &= \E{R_1} + \E{R_2} = \sum_{i: \Delta_i > 0}\Delta_i \cdot \E{L_{1,i}} + \E{R_2} \\
    & = \sum_{i: \Delta_i > 0}\Delta_i \cdot \tp{\Pr{\+S_i}\E{L_{1,i} \mid \+S_i} + \Pr{\ol{\+S}_i}\E{L_{1,i} \mid \ol{\+S}_i}}
    + \E{R_2} \\
    & \le \sum_{i: \Delta_i > 0}\Delta_i \cdot \frac{2c}{m}\cdot \tp{\frac{2\alpha}{e}}^{\frac{\alpha}{\alpha+1}}\cdot \tp{\frac{T}{n}}^\frac{1}{\alpha + 1} + \sum_{i: \Delta_i > 0}\Delta_i \cdot \frac{8\log T}{\Delta_i^2}\\
    & \overset{(\heartsuit)}{=} O\tp{\tp{\frac{2\alpha}{e}}^{\frac{\alpha}{\alpha+1}}  \frac{(n-m)T^{\frac{1}{\alpha + 1}}}{n^{1 + {\frac{1}{\alpha + 1}}}}\displaystyle\sum_{i:\Delta_i > 0}\Delta_i^{1 - 2\alpha}},
\end{align*}
where $(\heartsuit)$ is because $T$ is sufficiently large and $\forall 0< \Delta_i<1, \frac{1}{\Delta_i} \leq \Delta_i^{1-2\alpha}$.

\paragraph{Case 2:} The best arm $\arm_*$ is in the first $m$ arms. In this case, it is possible that $\arm_*$ does not belong to $M$ and we denote the best arm in $M$ as $\arm_{\!{king}}$. We first bound $\E{R_1}:$

\begin{align*}
    \E{R_1} & = \sum_{i: \Delta_i > 0} \Delta_i \cdot \tp{\Pr{\+S_i}\E{L_{1,i} \mid \+S_i} + \Pr{\ol{\+S}_i}\E{L_{1,i} \mid \ol{\+S}_i} }\\
    & \le \sum_{i: \Delta_i > 0} \Delta_i \cdot \frac{2c}{m} \cdot \tp{\frac{2\alpha}{e}}^{\frac{\alpha}{\alpha+1}} \cdot \tp{\frac{T}{n}}^\frac{1}{\alpha + 1} = O \tp{\tp{\frac{2\alpha}{e}}^{\frac{\alpha}{\alpha+1}}  \frac{(n-m)T^{\frac{1}{\alpha + 1}}}{n^{1 + {\frac{1}{\alpha + 1}}}}\displaystyle\sum_{i:\Delta_i > 0}\Delta_i}.
\end{align*}

Then we decompose $\E{R_2}$ according to whether $\arm_*$ belongs to $M$:
\begin{align}
        \E{R_2} &= \Pr{\+M_*}\E{R_2 \mid \+M_*} +\Pr{\ol{\+M}_*}\E{R_2 \mid \ol{\+M}_*} \notag \\
        & \le \E{R_2 \mid \+M_*} + \Pr{\ol{\+M}_*}\E{R_2 \mid \ol{\+M}_*} \notag \\
        & \overset{(\heartsuit)}{=}\E{R_2 \mid M_*} + \displaystyle\sum_{i\colon \arm_i \neq \arm_*}\Pr{\+B_i}\E{R_2 \mid \+B_i} \notag \\ 
        &\leq \sum_{i: \Delta_i > 0} \frac{8 \log T}{\Delta_i} + \displaystyle\sum_{i\colon \arm_i \neq \arm_*}\Pr{\+B_i}\E{R_2 \mid \+B_i},
\end{align}
where $(\heartsuit)$ holds since $\set{\+B_i \mid \arm_i \neq \arm_*}$ forms a partition of $\ol{\+M}_*$. For the term $\Pr{\+B_i}$, we have
\begin{align*}
    \Pr{\+B_i} & = \Pr{\+S_*\cap \+S_i\cap \+C_i\cap \+B_i} = \Pr{\+S_*\cap \+S_i}\Pr{\+C_i \mid \+S_*\cap \+S_i}\Pr{\+B_i \mid \+S_*\cap \+S_i\cap \+C_i} \\
    & \overset{(\spadesuit)}{\leq} \frac{{m - 2 \choose 2c - 2}}{{m \choose 2c}} \cdot \frac{1}{2c - 1} \cdot e^{-\frac{\Delta_i^2 L}{2}} \le \frac{4c}{m^2}\cdot e^{-\frac{\Delta_i^2 L}{2}},
\end{align*}
where $(\spadesuit)$ applies \Cref{lem:worse-win}. 

Then we turn to bound $\E{R_2\mid \+B_i}$ for every $i$ such that $\arm_i\ne \arm_*$. Since conditioned on $\+B_i$, $\arm_{\!{king}}\ne \arm_*$, the regret $\E{R_2\mid \+B_i}$ can be divided into three parts: those contributed by $\arm_{\!{king}}$, those $\arm_j$ with $\Delta_j$ close to $\Delta_{\!{king}}$ and those $\arm_j$ with $\Delta_j$ much larger than $\Delta_{\!{king}}$. We emphasize that the following inequalities involving conditional expectation $\E{\cdot\mid \+B_i}$ holds for all outcomes $\omega\in\+B_i$ (in the underlying probability space).

We have
\begin{align}
    \E{R_2\mid \+B_i}
    \le \Delta_{\!{king}}\cdot T + \sum_{j\colon \Delta_j > 0, \Delta_j\le 2\Delta_{\!{king}}} \Delta_j\cdot \E{L_{2,j} \mid \+B_i} + \sum_{j\colon \Delta_j>2\Delta_{\!{king}}} \Delta_j\cdot \E{L_{2,j} \mid \+B_i}.\label{eqn:R2-Bi-bound}
\end{align}
For those $\!{arm}_j$ with small $\Delta_j$, we have
\begin{align}
    \sum_{j\colon \Delta_j > 0, \Delta_j\le 2\Delta_{\!{king}}} \Delta_j\cdot \E{L_{2,j} \mid \+B_i} \le 2\Delta_{\!{king}}\cdot \sum_{j\colon \Delta_j\le 2\Delta_{\!{king}}} \E{L_{2,j}\mid \+B_i} \le 2\Delta_{\!{king}}\cdot T.\label{eqn:R2-Bi-jsmall}
\end{align}
For those $\!{arm}_j$ with large $\Delta_j$, we use \Cref{lem:ucb-sc} and obtain
\begin{align}
    \sum_{j\colon \Delta_j>2\Delta_{\!{king}}} \Delta_j\cdot \E{L_{2,j} \mid \+B_i} \le \sum_{j\colon \Delta_j>2\Delta_{\!{king}}} \Delta_j\cdot\frac{8\log T}{(\Delta_j-\Delta_{\!{king}})^2} \le \sum_{j\colon \Delta_j>2\Delta_{\!{king}}} \frac{32\log T}{\Delta_j}\label{eqn:R2-Bi-jlarge}.
\end{align}
Note that at every outcome $\omega\in\+B_i$, the random variable $\Delta_{\!{king}}(\omega)\le \Delta_i$ since $\!{king}$ is the best arm in $M$. Combining~(\ref{eqn:R2-Bi-bound}),~(\ref{eqn:R2-Bi-jsmall}) and~(\ref{eqn:R2-Bi-jlarge}), we obtain
\begin{align}
    \E{R_2\mid \+B_i} 
    \le 3\Delta_{\!{king}}\cdot T + \sum_{j\colon \Delta_j>2\Delta_{\!{king}}} \frac{32\log T}{\Delta_j} \le 3\Delta_{i}\cdot T + \sum_{j: \Delta_j > 0} \frac{32\log T}{\Delta_j}.\label{eqn:R2-Bi-all}
\end{align}
As a result,
\begin{align*}
    \sum_{i\colon \arm_i\ne \arm_*}\Pr{\+B_i}\E{R_2\mid \+B_i} &\le \sum_{i\colon \arm_i\ne \arm_*} \frac{4c}{m^2}\cdot e^{-\frac{\Delta_i^2 L}{2}} \tp{3\Delta_{i}\cdot T + \sum_{j: \Delta_j > 0} \frac{32\log T}{\Delta_j}} \\
    & \overset{(\diamondsuit)}{\le} \tp{\frac{2\alpha}{e}}^{\frac{\alpha}{\alpha+1}}\cdot\frac{12cn^{\frac{\alpha}{\alpha+1}}}{m^2} \displaystyle\sum_{i\colon arm_i \neq arm_*} \frac{T^{\frac{1}{\alpha+1}}}{\Delta_i^{2\alpha-1}} + \frac{128cn}{m^2} \sum_{j\colon \Delta_j > 0} \frac{\log T}{\Delta_j}\\
    & = O\tp{\tp{\frac{2\alpha}{e}}^{\frac{\alpha}{\alpha+1}} \frac{(n-m)T^{\frac{1}{\alpha + 1}}}{n^{1 + {\frac{1}{\alpha + 1}}}}\displaystyle\sum_{i:\Delta_i > 0}\Delta_i^{1 - 2\alpha} + \sum_{i: \Delta_i > 0}\frac{\log T}{\Delta_i}} \\
    & \overset{(\heartsuit)}{=} O \tp{\tp{\frac{2\alpha}{e}}^{\frac{\alpha}{\alpha+1}} \frac{(n-m)T^{\frac{1}{\alpha + 1}}}{n^{1 + {\frac{1}{\alpha + 1}}}}\displaystyle\sum_{i:\Delta_i > 0}\Delta_i^{1 - 2\alpha}},
\end{align*}
where ($\diamondsuit$) uses the inequality $e^{-x} \le \frac{\alpha^\alpha e^{-\alpha}}{x^\alpha}$ for $x>0$, ($\heartsuit$) is because $T$ is sufficiently large and $\forall 0< \Delta_i<1, \frac{1}{\Delta_i} \leq \Delta_i^{1-2\alpha}$.

In conclusion, we have $\E{R(T)} = \E{R_1} + \E{R_2} = O\tp{\tp{\frac{2\alpha}{e}}^{\frac{\alpha}{\alpha+1}} \frac{(n-m)T^{\frac{1}{\alpha + 1}}}{n^{1 + {\frac{1}{\alpha + 1}}}}\displaystyle\sum_{i:\Delta_i > 0}\Delta_i^{1 - 2\alpha}}$.
\end{proof}

\subsection{Small memory case \texorpdfstring{($2 \leq m < \frac{2}{3}n$)}{}} \label{sec:s-mem}
When $m < \frac{2}{3}n$, \Cref{alg:l-mem} is not feasible since the rest $n-m$ arms cannot be read into memory at once. We design \Cref{alg:s-mem} to deal with this case.

\begin{algorithm}[h]
\caption{Single-pass algorithm for MAB when $2 \leq m < \frac{2}{3}n$}
\label{alg:s-mem}
\Input{Time horizon $T$, memory size $m$, number of arms $n$ and a constant $\alpha \ge 1$.}
    \begin{algorithmic}[1]
        \State Let $L \leftarrow \tp{\frac{2\alpha}{e}}^{\frac{\alpha}{\alpha+1}}\cdot \tp{\frac{T}{m}}^{\frac{1}{\alpha+1}}$;
        \State Read the first $m-1$ arms into memory;
        \State Pull each of them $L$ times and calculate their empirical means respectively;
        \For{each arriving arm $\arm_i$}
            \State Choose an arm $\arm_j$ u.a.r. in the memory;\label{line:pickj}
            \State Read $\arm_i$ into memory;
            \State Pull $\arm_i$ $L$ times and calculate its empirical mean $\wh \mu_i$;
            \If{$\wh \mu_i > \wh \mu_j$}
                \State Discard $\arm_j$;
            \Else \State Discard $\arm_i$;
            \EndIf
        \EndFor
        \State Run UCB on all arms in the memory until the game ends; \label{line:ucb-s-mem}
        \Comment{The Exploitation Phase}
    \end{algorithmic}
\end{algorithm}

\Cref{alg:s-mem} will pull each arm $L$ times before the exploitation phase. Every incoming arm will be compared with a uniformly and randomly chosen arm in the memory. The worse of the two will be discarded to incorporate future new arms. This operation lasts until the end of the stream. Finally, we apply a standard UCB algorithm for arms in the memory until the end.


\begin{theorem}\label{thm:s-ub}
    Given any $\alpha \ge 1$ and any input instance, assuming $T$ is sufficiently large, \Cref{alg:s-mem} uses the memory of $m$ arms with expected regret
    $$
    \E{R(T)} = O\tp{\tp{\frac{2\alpha}{e}}^{\frac{\alpha}{\alpha+1}} \frac{T^{\frac{1}{\alpha + 1}}}{m^{\frac{1}{\alpha+1}}}\displaystyle\sum_{i:\Delta_i > 0}\Delta_i^{1 - 2\alpha}}.
    $$
\end{theorem}

\begin{proof}
We define notations $L_1, L_2, L_{1,i}, L_{2,i}, R_1, R_2$ similar to those in \Cref{sec:l-mem} with the only difference being that  $L_1$ now counts the number of rounds before the execution of \Cref{line:ucb-s-mem} in \Cref{alg:s-mem}. Denote $\arm_{\!{king}}$ as the best arm in the memory and $\+M_i$ as the event that $\arm_i$ is in the memory  during the exploitation phase. Conditioned on the event  $\ol{\+M}_*$ (namely the event that the best arm is \emph{not} in the memory during the exploitation phase), consider the maximal sequence of arms $\ab\{\arm_{s_i}\}_{i\in [\ell]}$ satisfying $\arm_{s_1} = \arm_*$ and $\arm_{s_{i+1}}$ beats $\arm_{s_i}$ for all $1\leq i \leq \ell-1$. By the maximality of the sequence, $\Pr{\+M_{s_\ell} \mid \ol{\+M}_*} = 1$, which deduces $\Delta_{\!{king}} \leq \Delta_{\ell}$ conditioned on $\ol{\+M}_*$.

For each $i\in [n]$, denote by $\+L_i$ the event that $\arm_i = \arm_{s_\ell}$, or equivalently the event that $\arm_i$ is the last one in the chain of elimination of $\arm_*$. Consider the time at which the $\arm_i$ is about to join the chain of elimination of $\arm_*$. There are two cases:
\begin{itemize}
    \item the $\arm_*$ is the incoming arm and at \Cref{line:pickj}, the algorithm picks $\arm_i$;
    \item or the $\arm_i$ is the incoming arm and at \Cref{line:pickj}, the algorithm picks an arm already in the chain.
\end{itemize}
Moreover, in both cases, we must have that $\wh\mu_i > \wh\mu_*$ in order for $\+L_i$ to happen. As a result, we have
\[
    \Pr{\+L_i} \leq \frac{1}{m-1}\cdot e^{-\frac{\Delta_i^2 L}{2}}.
\]

We first bound $\E{R_1}$,
\begin{align*}
    \E{R_1} 
    &= \sum_{i:\Delta_i > 0} \Delta_i \cdot L = \sum_{i:\Delta_i > 0} \Delta_i \cdot \tp{\frac{2\alpha}{e}}^{\frac{\alpha}{\alpha+1}}\cdot\tp{\frac{T}{m}}^{\frac{1}{\alpha+1}} = O\tp{\tp{\frac{2\alpha}{e}}^{\frac{\alpha}{\alpha+1}} \tp{\frac{T}{m}}^{\frac{1}{\alpha+1}}\sum_{i:\Delta_i > 0}\Delta_i}.
\end{align*}

Then we also decompose $\E{R_2}$ based on whether $\+M_*$ happens. 
\begin{align}
    \E{R_2} &= \Pr{\+M_*}\E{R_2 \mid \+M_*} +\Pr{\ol{\+M}_*}\E{R_2 \mid \ol{\+M}_*} \notag \\
    & \le \E{R_2 \mid \+M_*} + \Pr{\ol{\+M}_*}\E{R_2 \mid \ol{\+M}_*} \notag \\
    & \overset{(\clubsuit)}{=}\E{R_2 \mid M_*} + \displaystyle\sum_{i\colon \arm_i \neq \arm_*}\Pr{\+L_i}\E{R_2 \mid \+L_i} \notag \\ 
    &\leq \sum_{i: \Delta_i > 0} \frac{8 \log T}{\Delta_i} +  \frac{1}{m-1} \displaystyle\sum_{i\colon \arm_i \neq \arm_*}e^{-\frac{\Delta_i^2 L}{2}} \E{R_2 \mid \+L_i}, 
\end{align}
where $(\clubsuit)$ holds since $\set{\+L_i \mid \arm_i \neq \arm_*}$ forms a partition of $\ol{\+M}_*$. Conditioned on $\+L_i$, $\Delta_{\!{king}} \leq \Delta_i$, thus we have following holds:
\begin{align*}
    \E{R_2\mid \+L_i}
    &\le \Delta_{\!{king}}\cdot T + \sum_{j\colon \Delta_j > 0, \Delta_j\le 2\Delta_{\!{king}}} \Delta_j\cdot \E{L_{2,j} \mid \+L_i} + \sum_{j\colon \Delta_j>2\Delta_{\!{king}}} \Delta_j\cdot \E{L_{2,j} \mid \+L_i} \\
    &\le 3\Delta_{\!{king}}\cdot T + \sum_{j\colon \Delta_j>2\Delta_{\!{king}}} \Delta_j\cdot\frac{8\log T}{(\Delta_j-\Delta_{\!{king}})^2} \\
    &\le 3\Delta_i\cdot T + \sum_{j: \Delta_j > 0} \frac{32\log T}{\Delta_j}. 
\end{align*}

Therefore,
\begin{align*}
    \E{R_2} &\le  \sum_{i: \Delta_i > 0} \frac{8 \log T}{\Delta_i} +  \frac{1}{m-1} \displaystyle\sum_{i\colon \arm_i \neq \arm_*}e^{-\frac{\Delta_i^2 L}{2}}\cdot 3\Delta_i T + \frac{1}{m-1} \displaystyle\sum_{i\colon \arm_i \neq \arm_*} \sum_{j: \Delta_j > 0} \frac{32\log T}{\Delta_j} \\
    & \overset{(\diamondsuit)}{\le}\sum_{i: \Delta_i > 0} \frac{8 \log T}{\Delta_i} + \tp{\frac{2\alpha}{e}}^{\frac{\alpha}{\alpha+1}}\cdot \frac{3m^{\frac{\alpha}{\alpha+1}}T^{\frac{1}{\alpha+1}}}{m-1} \displaystyle\sum_{i\colon arm_i \neq arm_*} \Delta_i^{1-2\alpha} + \frac{32n}{m-1} \sum_{j\colon \Delta_j > 0} \frac{\log T}{\Delta_j}\\
    &= O\tp{\tp{\frac{2\alpha}{e}}^{\frac{\alpha}{\alpha+1}} \frac{T^{\frac{1}{\alpha + 1}}}{m^{\frac{1}{\alpha+1}}}\displaystyle\sum_{i:\Delta_i > 0}\Delta_i^{1 - 2\alpha} + \frac{n}{m}\sum_{i: \Delta_i > 0}\frac{\log T}{\Delta_i}} \\
    & \overset{(\spadesuit)}{=} O \tp{\tp{\frac{2\alpha}{e}}^{\frac{\alpha}{\alpha+1}} \frac{T^{\frac{1}{\alpha + 1}}}{m^{{\frac{1}{\alpha + 1}}}}\displaystyle\sum_{i:\Delta_i > 0}\Delta_i^{1 - 2\alpha}},    
\end{align*}
where ($\diamondsuit$) uses the inequality $e^{-x} \le \frac{\alpha^\alpha e^{-\alpha}}{x^\alpha}$ for $x>0$ and ($\spadesuit$) is because $T$ is sufficiently large and $\forall 0< \Delta_i<1, \frac{1}{\Delta_i} \leq \Delta_i^{1-2\alpha}$.

In conclusion, we have $\E{R(T)} = \E{R_1} + \E{R_2} =  O\tp{\tp{\frac{2\alpha}{e}}^{\frac{\alpha}{\alpha+1}} \frac{T^{\frac{1}{\alpha + 1}}}{m^{\frac{1}{\alpha+1}}}\displaystyle\sum_{i:\Delta_i > 0}\Delta_i^{1 - 2\alpha}}$.

\end{proof}

\section{Gap-Dependent Regret Lower Bounds}\label{sec:lb}

In this section, we will prove the regret lower bounds for every $2\le m<n$. We first provide our construction of the hard instances in \Cref{sec:hard-instances}. Then we reduce the problem of minimizing regret on these instances to the problem of best arm retention, whose sample complexity lower bounds have recently been established. Here we propose our lower bound of regret.

\begin{theorem}\label{thm:lb}
    Given any $\alpha \ge 1$, for any integer $k$ satisfying $m\leq k < n$ and any algorithm $\+A$, assuming $T$ is sufficiently large, there exists a set of hard instances on which the expected regret of $\+A$ is $\Omega\tp{16^{-\alpha}\cdot\frac{(k-m+1) T^{\frac{1}{\alpha+1}}}{k^{1 + \frac{1}{\alpha+1}}} \sum_{i:\Delta_i > 0}\Delta_i^{1-2\alpha}}$.
\end{theorem}

By picking $k=\frac{3}{2}m$ when $m<\frac{2}{3}n$ and $k=n-1$ otherwise, we obtain
\begin{corollary}
    Given any $\alpha \ge 1$, for any algorithm $\+A$ using $2\le m<n$ memory on $n$ arms, assuming $T$ is sufficiently large, there always exists a set of hard instances on which the expected regret of $\+A$ is $\Omega\tp{16^{-\alpha}\cdot\frac{T^{\frac{1}{\alpha+1}}}{m^{\frac{1}{\alpha+1}}} \sum_{i:\Delta_i > 0}\Delta_i^{1-2\alpha}}$ when $m < \frac{2}{3}n$ and $\Omega\tp{16^{-\alpha}\cdot\frac{(n-m) T^{\frac{1}{\alpha+1}}}{k^{1 + \frac{1}{\alpha+1}}} \sum_{i:\Delta_i > 0}\Delta_i^{1-2\alpha}}$ when $m\ge \frac{2}{3}n$.
\end{corollary}

\subsection{The construction of the hard instances}\label{sec:hard-instances}

In this section, we provide the family of hard instances and prove in the next section that any algorithm exhibits large regret on at least \emph{one of} the instances in the family. 

Given a constant $\alpha \ge 1$, let $T$ be sufficiently large and $\eps = \frac{1}{4}\cdot \tp{\frac{k}{T}}^{\frac{1}{2+2\alpha}}$. For every $m\le k<n$, the hard instances for our problem are as follows:
\begin{center}
$
    \@I:\left\{
        \begin{array}{ll}
          \nu_1 = \Big(\underbrace{\frac{1}{2} + n\eps, \frac{1}{2} + (n-1)\eps, \frac{1}{2} + (n-1)\eps, \cdots, \frac{1}{2} + (n-1)\eps}_{k\mbox{ arms}}, \frac{1}{2}, \cdots, \frac{1}{2}\Big)\\
          \nu_2 = \Big(\frac{1}{2} + n\eps, \frac{1}{2} + (n+1)\eps, \frac{1}{2} + (n-1)\eps, \cdots, \frac{1}{2} + (n-1)\eps, \frac{1}{2}, \cdots, \frac{1}{2}\Big)\\
          \quad\quad\vdots\\
          \nu_k = \Big(\frac{1}{2} + n\eps, \frac{1}{2} + (n-1)\eps,\cdots, \frac{1}{2} + (n-1)\eps, \frac{1}{2} + (n+1)\eps, \frac{1}{2}, \cdots, \frac{1}{2}\Big)\\
        \end{array}
      \right.    
$
$
    \@I':\left\{
        \begin{array}{ll}
          \nu_1' = \Big(\underbrace{\frac{1}{2} + n\eps, \frac{1}{2} + (n-1)\eps, \frac{1}{2} + (n-1)\eps, \cdots, \frac{1}{2} + (n-1)\eps}_{k\mbox{ arms}}, 1, \cdots, 1\Big)\\
          \nu_2' = \Big(\frac{1}{2} + n\eps, \frac{1}{2} + (n+1)\eps, \frac{1}{2} + (n-1)\eps, \cdots, \frac{1}{2} + (n-1)\eps, 1, \cdots, 1\Big)\\
          \quad\quad\vdots\\
          \nu_k' = \Big(\frac{1}{2} + n\eps, \frac{1}{2} + (n-1)\eps,\cdots, \frac{1}{2} + (n-1)\eps, \frac{1}{2} + (n+1)\eps, 1, \cdots, 1\Big)\\
        \end{array}
      \right.     
$
\end{center}

In other words, we have two families of instances, $\@I$ and $\@I'$, each consisting of $k$ instances. The first arm of each instance has $\!{Ber}\tp{\frac{1}{2}+n\eps}$ reward. For each $2\le i\le k$, the $i$-th arm of $\nu_i$ and $\nu_i'$ has reward $\!{Ber}\tp{\frac{1}{2}+(n+1)\eps}$ and the remaining arms among $\arm_2$ to $\arm_k$ has reward $\!{Ber}\tp{\frac{1}{2}+(n-1)\eps}$. The difference between $\nu_i$ and $\nu_i'$ is the last $n-k$ arms where in $\nu_i$ they all have reward $\!{Ber}\tp{\frac{1}{2}}$ while in $\nu_i'$ the rewards are $\!{Ber}(1)$.

We will derive lower bounds for our streaming algorithm from the sample complexity lower bounds for the BAR problem. We also specify our hard instances for the BAR problem. Let $\nu_i[1:k]$ represent the first $k$ arms in $\nu_i$ and let $\rho_i=\nu_i[1:k]$. In other words,

\[
\rho_i = \Big(\frac{1}{2} + n\eps, \frac{1}{2} + (n-1)\eps,\cdots, \underbrace{\scriptsize \frac{1}{2} + (n+1)\eps}_{\mbox{\scriptsize the $i$-th arm}}, \cdots,\frac{1}{2} + (n-1)\eps \Big).
\]

We construct the set of instances $\@I_0 = \set{\rho_i = \nu_i[1:k], i\in [k]}$. Then we have the following sample complexity for BAR on $\@I_0$.

\begin{lemma}[implicitly in \cite{CHZ24}] \label{lem:bar-lb}
    For any $(\eps,\delta)$-PAC algorithm for the BAR problem on instances in $\@I_0$ such that $\eps \leq \frac{1}{8(n-1)}$ and $\delta \leq \frac{k-m}{k} \tp{1-\beta}$, where $\beta \in (0,1)$ is an arbitrary constant, its sample times $T$ on the input $\rho_1$ satisfies 
    $
    \E[\rho_1]{T} \ge \frac{\beta}{32}\cdot \frac{k-m-\delta}{\eps^2}\log{\frac{k-m-\delta}{(k-1)\delta}}.
    $
\end{lemma}

The same sample complexity lower bound has been proved in ~\cite{CHZ24} for a similar hard instance family and their proof can be adapted to our hard instances. We will provide a proof of \Cref{lem:bar-lb} in \Cref{sec:bar-sc-app} for the sake of self-containment.

\subsection{The proof of the lower bounds}

For every instance $\nu\in \@I\cup \@I'$, the execution of an algorithm can be divided into two stages: rounds before reading the $(k+1)$-th arm and the rounds of and after reading the $(k+1)$-th arm. We call them stage one and stage two respectively. We use $L_1$ and $L_2$ to denote the number of rounds for the two stages and use $R_1$ and $R_2$ to denote the regret incurred in the two stages respectively. Therefore $L_1+L_2=T$. Do not confuse the notations here with ones in the upper bound proof since we are dealing with any algorithm here.

Our lower bounds proof is by formalizing the following dilemma faced by each algorithm: For instances in $\@I$,
\begin{itemize}
    \item $L_1$ cannot be too large, otherwise, their counterparts in $\@I'$ will incur large $R_1$, 
    \item and $L_1$ cannot be too small either since otherwise the algorithm cannot identify the best arm among the first $k$ arms, which will cause $R_2$ large for some instances.
\end{itemize}
For the second point above, we reduce the lower bound to the sample complexity of the best arm retention problem.

Let $f(k,m)=\frac{2}{16^{\alpha+1}}\cdot\frac{k-m+1}{k^{\frac{1}{\alpha+1}}}$. If there exists an algorithm $\+A$ with regret at most $\frac{2}{16^{\alpha+1}}\cdot\frac{(k-m+1)T^{\frac{1}{\alpha+1}}}{k^{1 + \frac{1}{\alpha+1}}} \cdot \allowbreak\sum_{i\colon \Delta_i>0} \Delta_i^{1-2\alpha}$, since $\alpha \ge 1$, then it holds that 
\begin{align*}
    \E[\nu_1]{R(T)} 
    &\le \frac{2}{16^{\alpha+1}}\cdot\frac{(k-m+1)T^{\frac{1}{\alpha+1}}}{k^{1 + \frac{1}{\alpha+1}}} \tp{(k-1)\eps^{1-2\alpha} + (n-k)(n\eps)^{1-2\alpha}}\\
    &\le \frac{2}{16^{\alpha+1}}\cdot\frac{(k-m+1)T^{\frac{1}{\alpha+1}}}{k^{1 + \frac{1}{\alpha+1}}}\cdot k\eps^{1-2\alpha} = f(k,m)\cdot T^{\frac{1}{\alpha+1}}\cdot\eps^{1-2\alpha},
\end{align*}
and for any $2\le i\le k$,
\begin{align*}
    \E[\nu_i]{R(T))} &\le \frac{2}{16^{\alpha+1}}\cdot\frac{(k-m+1)T^{\frac{1}{\alpha+1}}}{k^{1 + \frac{1}{\alpha+1}}} \tp{\eps^{1-2\alpha} + (k-2)(2\eps)^{1-2\alpha} + (n-k)((n+1)\eps)^{1-2\alpha}} \\
    &\le \frac{2}{16^{\alpha+1}}\cdot\frac{(k-m+1)T^{\frac{1}{\alpha+1}}}{k^{1 + \frac{1}{\alpha+1}}}\cdot k\eps^{1-2\alpha} = f(k,m)\cdot T^{\frac{1}{\alpha+1}}\cdot \eps^{1-2\alpha},
\end{align*}
Similarly for any $\nu'\in \@I'$ and sufficiently large $T$,
\begin{align*}
    \E[\nu']{R(T)} 
    &\le \frac{2}{16^{\alpha+1}}\cdot\frac{(k-m+1)T^{\frac{1}{\alpha+1}}}{k^{1 + \frac{1}{\alpha+1}}}\cdot k\tp{\frac{1}{2} - (n+1)\eps}^{1-2\alpha} \\
    &\leq  \frac{2}{16^{\alpha+1}}\cdot\frac{(k-m+1)T^{\frac{1}{\alpha+1}}}{k^{\frac{1}{\alpha+1}}}\cdot \tp{\frac{1}{4}}^{1-2\alpha} = \frac{1}{32}\cdot f(k,m)\cdot T^{\frac{1}{\alpha+1}}.       
\end{align*}

The above discussions are summarized below.
\begin{lemma}\label{lem:regret-ub}
If there exists an algorithm $\+A$ with regret at most $\frac{2}{16^{\alpha+1}}\cdot\frac{(k-m+1)T^{\frac{1}{\alpha+1}}}{k^{1 + \frac{1}{\alpha+1}}} \sum_{i\colon \Delta_i>0} \Delta_i^{1-2\alpha}$ on each instance in $\@I\cup\@I'$, then for every $\nu\in \@I$, $\nu'\in \@I'$, 
\begin{equation}\label{eqn:small-regret}
    \E[\nu]{R(T)}\le f(k,m)T^{\frac{1}{\alpha+1}}\eps^{1-2\alpha},\;\E[\nu']{R(T)}\le \frac{1}{32}\cdot f(k,m)\cdot T^{\frac{1}{\alpha+1}}
\end{equation}
where the randomness in the expectation is from both the (possible) randomness of $\+A$ and the randomness of the instance. 
\end{lemma}

Provided the upper bound on the regret, we now show that $L_1$ cannot be too large.

\begin{lemma}\label{lem:L1-small}
    If there exists an algorithm $\+A$ satisfying \cref{eqn:small-regret}, then for \emph{every} $i\in [k]$,
    $
    \E[\nu_i]{L_1} \le  \frac{1}{8}f(k,m) T^{\frac{1}{\alpha+1}}.
    $
\end{lemma}
\begin{proof}
    Note that the random variable $L_1$ only depends on the performance of the first $k$ arms in the stream, and the first $k$ arms of $\nu_i$ are the same as that of $\nu_i'$. Therefore $\E[\nu_i']{L_1} = \E[\nu_i]{L_1}$. Each pull of the first $k$ arms in $\nu_i'$ incurs a regret at least $\frac{1}{2}-(n+1)\eps$. So we have
    \[
        \E[\nu_i']{L_1}\cdot \tp{\frac{1}{2}-(n+1)\eps}\le \E[\nu_i']{R}\le \frac{1}{32}\cdot f(k,m)\cdot T^{\frac{1}{\alpha+1}}.
    \]
    This implies that
    \[
    \E[\nu_i]{L_1}=\E[\nu_i']{L_1}\le \frac{\frac{1}{32}f(k,m)T^{\frac{1}{\alpha+1}}}{(1/2-(n+1)\eps)}\le \frac{1}{8}\cdot f(k,m)\cdot T^{\frac{1}{\alpha+1}},
    \]
    for sufficiently large $T$.
\end{proof}

On the other hand, we prove the following lemma, justifying that for some $\nu_i$, $\E[\nu_i]{L_1}$ cannot be small provided the algorithm has small regret. 

\begin{lemma}\label{lem:L1-large}
    If there exists an algorithm $\+A$ satisfying \cref{eqn:small-regret}, then,
    $
        \E[\nu_1]{L_1} > 8\cdot f(k,m)\cdot T^{\frac{1}{\alpha+1}}.
    $
\end{lemma}

Clearly \Cref{thm:lb} holds by combining \Cref{lem:regret-ub}, \Cref{lem:L1-small} and \Cref{lem:L1-large}.

\bigskip
The remaining part of the section devotes to a proof of \Cref{lem:L1-large}. For every $i\in [k]$, we let $\tau_i$ be the event that ``the arm $i$ is not in the memory at the beginning of stage two''. We now show that for every $\nu_i\in \@I$, provided the algorithm $\+A$ has small regret, $\Pr[\nu_i]{\tau_i}$ is small. 

\begin{lemma}
    For every $i\in [k]$ and sufficiently large $T$, it holds that $\Pr[\nu_i]{\tau_i}\le \frac{1}{2}\cdot\frac{k-m+1}{k}$.
\end{lemma}
\begin{proof}
    For every $i\in [k]$, it follows from \Cref{lem:L1-small} that 
    \[
        \E[\nu_i]{L_2}= T-\E[\nu_i]{L_1}\ge T-\frac{1}{8} f(k,m)T^{\frac{1}{\alpha+1}} \ge \frac{T}{4}.
    \]
    Note that once $\tau_i$ happens on $\nu_i$, each pull in stage two incurs a regret at least $\eps$. Therefore, for every $i\in [k]$,
    \begin{align*}
        \E[\nu_i]{R}
        &\ge \Pr[\nu_i]{\tau_i}\E[\nu_i]{R_2\mid \tau_i} \ge \Pr[\nu_i]{\tau_i}\cdot \frac{T}{4}\cdot \eps.
    \end{align*}
    It then follows from \cref{eqn:small-regret} that for sufficiently large $T$,
    $\Pr[\nu_i]{\tau_i} \le 4f(k,m)T^{-\frac{\alpha}{\alpha+1}} \eps^{-2\alpha} = \frac{1}{2}\cdot\frac{k-m+1}{k}$.
\end{proof}

Now we claim that the algorithm $\+A$ can be used to solve the best arm retention problem with instances in $\@I_0$ that retain $m-1$ arms out of $k$ arms: Given an instance for BAR in $\@I_0$, we put them in a stream and call the algorithm $\+A$. At the end of stage one, we output all the arms in the memory. Note that the memory is $m-1$ instead of $m$ here since we do not count the one retaining the $(k+1)$-th arm in the stream.

Clearly, the algorithm above is a $(\delta,\eps)$-PAC algorithm with $ \delta=\Pr[\nu_i]{\tau_i}\le \frac{1}{2}\cdot\frac{(k-m+1)}{k}$ on these instances. On the other hand, by picking $\beta=\frac{1}{2}$, it follows from \Cref{lem:bar-lb} that
\begin{align*}
    \E[\nu_1]{L_1} &\ge \frac{k-m+1-\delta}{64\eps^2} \log \frac{k-m+1-\delta}{(k-1)\delta} \ge \frac{(k-m+1)-\frac{(k-m+1)}{2k}}{64\eps^2} \log \frac{(k-m+1) - \frac{(k-m+1)}{2k}}{(k-1)\cdot\frac{(k-m+1)}{2k}} \\
    &= \tp{\frac{2k-1}{k}\log \frac{2k-1}{k-1}\cdot 16^\alpha} f(k,m)T^{\frac{1}{\alpha+1}} > 8\cdot f(k,m)\cdot T^{\frac{1}{\alpha+1}}.
\end{align*}

%
%

\paragraph{Acknowledgement} 
The authors would like to thank Yuchen He for the help at various stages of the work.

%
%
\bibliographystyle{alpha}
\bibliography{arxiv_version}






\appendix
\onecolumn

\section{Omitted Proofs in \texorpdfstring{\Cref{sec:ub}}{}}
\subsection{Technical preliminaries} \label{sec:ub-prelim}

\begin{lemma}[Hoeffding's inequality, \cite{MU17}]\label{lem:Hoef-ineq}
    Let $X_1,\cdots,X_n$ be $n$ independent random variables defined on a common probability space and taking values in $[a,b]$. Define $X \defeq \sum_{i=1}^n X_i$, then for any $s>0$,
    $$
        \Pr{X-\E{X} \geq s} \leq \exp \tp{-\frac{2s^2}{n(b-a)^2}}.
    $$
\end{lemma}
Applying Hoeffding's inequality, we could get the following lemma:

\begin{lemma}
    Let $\arm_1$ and $\arm_2$ be two different arms with reward means of $\mu$ and $\mu+\Delta, \Delta>0$. Suppose we sample each arm $L$ times and obtain empirical mean of $\wh{\mu}_1$ and $\wh\mu_2$, then 
    $$
        \Pr{\wh\mu_1 \ge \wh{\mu}_2} \leq e^{-\frac{L\Delta^2}{2}}.
    $$
\end{lemma}

\begin{proof}
Denote the reward of pulling $\arm_1$ and $\arm_2$ in the $i$-th round is $X_i$ and $Y_i$, then
    \begin{align*}
         \Pr{\wh{\mu}_1 \ge \wh{\mu}_2} 
        &= \Pr{\frac{1}{L}\sum_{i=1}^L X_i \ge \frac{1}{L}\sum_{i=1}^L Y_i} \\
        &= \Pr{\sum_{i=1}^L \tp{X_i - Y_i} - (-L\Delta) \ge L\Delta} \\
        \mr{\Cref{lem:Hoef-ineq}}
        &\leq \exp\tp{- \frac{2(L\Delta)^2}{4L}} = e^{- \frac{L\Delta^2}{2}}.
    \end{align*}
\end{proof}

\subsection{The Upper Confidence Bound algorithm} \label{sec:ucb-app}
We provide the description of the UCB algorithm used in \Cref{alg:l-mem} and \Cref{alg:s-mem} here. For more detailed information, please refer to the work of \cite{LS20}.

Given the parameter $\delta$, denote $T_i(t-1)$ as the number of times pulling $\arm_i$ before round $t$, and $\wh{\mu}_i(t-1)$ is the empirical mean of $\arm_i$ before round $t$. We define the upper confidence bound for $\arm_i$ in round $t$ as
$$
    \!{UCB}_i(t-1, \delta) = 
        \begin{cases}
            \infty, \quad & T_i(t-1) = 0 \\
            \wh{\mu}_i(t-1) + \sqrt{\frac{2\log \frac{1}{\delta}}{T_i(t-1)}}  , \quad & otherwise \\
        \end{cases}
$$

\begin{algorithm}[H]
\caption{UCB($\delta$)}
\label{alg:s-mem-app}
\Input{Time horizon $T$, a set of $n$ arms, probability $\delta$.}
    \begin{algorithmic}[1]
        \For{$t = 1,\cdots, T$}
            \State Choose $A_t = \argmax_{i \in [n]} \!{UCB}_i(t-1,\delta)$ and pull it once;
            \State Observe its reward and update all arms' upper confidence bounds;
        \EndFor
    \end{algorithmic}
\end{algorithm}

\section{Omitted Proofs in \Cref{sec:lb}} \label{sec:bar-sc-app}

\subsection{Technique preliminaries}
The Kullback-Leibler (KL) divergence between $\Ber{x}$ and $\Ber{y}$ is defined as
\[
    d(x,y) \defeq x\log\tp{\frac{x}{y}} + (1-x)\log\tp{\frac{1-x}{1-y}},
\]
Here we give some propositions of the KL-divergence.

\begin{proposition} \label{pro:KL-taylor}
    Given $\frac{1}{2} \leq \mu \leq \frac{5}{8}$, $0 < \zeta \leq \frac{1}{8}$, then $d(\mu, \mu+\zeta) \leq 4\zeta^2$.
\end{proposition}

\begin{proof}
    Applying Taylor expansion, there exist $\eta_1 \in \tp{0,\frac{\zeta}{\mu}}, \eta_2 \in \tp{-\frac{\zeta}{1 - \mu}, 0}$ such that
    \begin{align*}
        d(\mu, \mu + \zeta)
        &= \mu \cdot \log \frac{\mu}{\mu + \zeta} + (1-\mu)\cdot \log \frac{1-\mu}{1-\mu-\zeta} \\
        &= -\mu\tp{ \frac{\zeta}{\mu} - \frac{1}{2}\cdot\frac{\zeta^2}{\mu^2} + \frac{1}{3(1+ \eta_1)^3}\cdot \frac{\zeta^3}{\mu^3}} \\
        &- (1-\mu)(-\frac{\zeta}{1-\mu} - \frac{1}{2}\frac{\zeta^2}{(1-\mu)^2} - \frac{1}{3(1+ \eta_2)^3}\cdot \frac{\zeta^3}{(1 -\mu)^3}) \\
        &\leq \zeta^2\tp{\frac{1}{2\mu} + \frac{1}{2(1-\mu)}} + \frac{\zeta^3}{3} \tp{\frac{1-\mu}{(1-\mu-\zeta)^3} - \frac{\mu}{(\mu + \zeta)^3}} \\
        &\leq \frac{7}{3}\zeta^2 + \frac{\zeta^2}{24}\cdot \frac{832}{27} \leq 4\zeta^2.
    \end{align*}
\end{proof}

\begin{proposition} \label{pro:KL-conv}
    $d(x,y)$ is convex for $x$ (or $y$) when $y$ (or $x$) is fixed.
\end{proposition}

\begin{proof}
\[
    \frac{\partial^2 d(x,y)}{\partial x^2} = \frac{\partial \log\tp{\frac{x(1-y)}{(1-x)y}}}{\partial x} = \frac{1}{x} + \frac{1}{1-x} \ge 0.
\]
\[
    \frac{\partial^2 d(x,y)}{\partial y^2} = \frac{\partial \tp{-\frac{x}{y} + \frac{1-x}{1-y}}}{\partial y} = \frac{x}{y^2} + \frac{1-x}{(1-y)^2} \ge 0.
\]
\end{proof}

\begin{proposition} \label{pro:KL-ineq}
    $\forall 0 \leq p \leq x \leq y \leq q \leq 1$,
    \[
        d(p,q) \ge d(x,y).
    \]
\end{proposition}

\begin{proof}
    First we have $\frac{\partial d(x,y)}{\partial x} \big|_{x=y} = 0$. By \Cref{pro:KL-conv}, $d(\cdot, y)$ achieves the minimum when $x = y$. So $d(p, y) \ge  d(x,y)$. Similarly $d(x,\cdot)$ achieves the minimum when $y=x$. Therefore $d(p,q) \ge d(p,y) \ge d(x,y)$.
\end{proof}

\begin{lemma} \label{lem:sum-ineq}
    $\forall x_1, x_2, \cdots, x_N \in [0,1]$ with average $a \defeq \frac{\sum_{i=1}^N x_i}{N} < b\in [0,1]$, then:
    \[
        \sum_{i:x_i < b}d(x_i, b) \ge N\cdot d(a,b).
    \]
\end{lemma}

\begin{proof}
    Let $S = \set{i: x_i < b}$, and $\abs{S} = s$. According to the definition,
    \[
        N\cdot a = \sum_{i\in S} x_i + \sum_{i\in [N]\backslash S}x_i \ge \sum_{i\in S} x_i + (N-s)\cdot b.
    \]
    By \Cref{pro:KL-ineq}, $\forall x \leq y \leq b$, $d(x,b) \ge d(y,b)$. By \Cref{pro:KL-conv}, $d(x,b)$ is convex for $x$. Using Jensen's inequality for $d(\cdot, b)$ and we have:
    \[
        \sum_{i\in S}d(x_i, b) \ge s\cdot d\tp{\frac{\sum_{i\in S} x_i}{s}, b} \ge s\cdot d\tp{\frac{N\cdot a - (N-s)\cdot b}{s}, b}.
    \]
    Note that
    \[
        \frac{s}{N}\cdot \frac{N\cdot a - (N-s)\cdot b}{s} + \frac{N-s}{N}\cdot b = a.
    \]
    Using the convexity of $d(\cdot, b)$ again, we have
    \[
        \frac{s}{N}\cdot d\tp{\frac{N\cdot a - (N-s)\cdot b}{s}, b} + \frac{N-s}{N}\cdot d(b,b)
        = \frac{s}{N}\cdot d\tp{\frac{N\cdot a - (N-s)\cdot b}{s}, b} \ge d(a,b),
    \]
    which deduces that $\sum_{i:x_i < b}d(x_i, b) \ge N\cdot d(a,b)$.
\end{proof}

\begin{lemma} \label{lem:KL-first-item-ineq}
    $\forall 0<q<p<1$, let $r \defeq \frac{p-q}{q}$, then 
    \[
       d(p,q) \ge \frac{r}{2+2r}\cdot p\log \frac{p}{q}.
    \]
\end{lemma}

\begin{proof}
    We first give some inequalities from \cite{Top07},
    \begin{enumerate}
        \item $\log\tp{1+x} \ge \frac{2x}{2+x}, \forall x>0.$ \label{eqn:ineq-1}
        \item $\log\tp{1+x} \ge \frac{x(2+x)}{2(1+x)}, \forall -1<x\leq 0.$ \label{eqn:ineq-2}
    \end{enumerate}

    Applying \cref{eqn:ineq-1} and we could get
    \begin{align*}
        p\log \frac{p}{q} = p\log\tp{1 + \frac{p-q}{q}} \ge p\cdot \frac{2(p-q)}{p+q} = (p-q)\cdot\tp{1 + \frac{p-q}{p-q+2q}} =(p-q)\tp{1+\frac{r}{2+r}}.
    \end{align*}

    Applying \cref{eqn:ineq-2} and we could get
    \begin{align*}
        (1-p)\log \frac{1-p}{1-q} = (1-p)\log\tp{1 + \frac{q-p}{1-q}} \ge  \frac{(q-p)(2-q-p)}{2(1-q)} \ge -(p-q).
    \end{align*}

    Then we could bound $d(p,q)$ as follows,
    \begin{align*}
        d(p,q) = p\log \frac{p}{q} + (1-p)\log \frac{1-p}{1-q} &= \frac{r}{2+2r}\cdot p\log \frac{p}{q} + \frac{2+r}{2+2r}\cdot p\log \frac{p}{q} + (1-p)\log\tp{\frac{1-p}{1-q}} \\
        &\ge \frac{r}{2+2r} \cdot p\log \frac{p}{q} + \frac{2+r}{2+2r}\cdot (p-q)\tp{1+\frac{r}{2+r}} - (p-q) \\
        &= \frac{r}{2+2r} \cdot p\log \frac{p}{q} 
    \end{align*}
    
\end{proof}

Now we give an important lemma which is used to bound the exploring times in order to extinguish two instances.

\begin{lemma}[\cite{KCG16}] \label{lem:KL-pull-lb}
    For any two MAB instances $\rho = \tp{\mu_1,\cdots, \mu_k}, \rho' = \tp{\mu_1', \cdots, \mu_k'}$ with $k$ arms, and for any algorithm with almost-surely finite stopping time $T$, and event $\tau \in \+F_T$,
    \[
        \sum_{i=1}^k \tp{\E[\rho]{T_i}\cdot d(\mu_i, \mu_i')} \ge d\tp{\Pr[\rho]{\tau}, \Pr[\rho']{\tau} }.
    \]
\end{lemma}

\subsection{Proof of \texorpdfstring{\Cref{lem:bar-lb-app}}{}}
Let instances $\@I_0 = \set{\rho_i, i\in[k]}$ containing $k$ arms, where 
$$
    \rho_1 = \tp{\frac{1}{2}+ n\eps, \frac{1}{2}+(n-1)\eps, \cdots, \frac{1}{2}+(n-1)\eps}
$$ for $ n \ge 3$. For $1<i\leq k$, $\rho_i$ differs from $\rho_1$ only in the $i$-th arm: $\rho_i(i) = \frac{1}{2}+(n+1)\eps$. Then, we have the following lemma.

\begin{lemma}[\Cref{lem:bar-lb} restated] \label{lem:bar-lb-app}
    For any $(\eps,\delta)$-PAC algorithm $\+A$ which addresses the BAR problem on instances in $\@I_0$  such that $\eps \leq \frac{1}{8(n-1)}$ and $\delta \leq \frac{k-m}{k} \tp{1-\beta}$, where $\beta \in (0,1)$ is an arbitrary constant, its sample complexity $T$ on the input $\rho_1$ satisfies 
    $$
    \E[\rho_1]{T} \ge \frac{\beta}{32}\cdot \frac{k-m-\delta}{\eps^2}\log{\frac{k-m-\delta}{(k-1)\delta}}.
    $$
\end{lemma}

Here, we only take into account algorithms that have an almost-surely finite stopping time. If this condition is not met, the sample complexity becomes infinite, and the lemma is trivially satisfied. We will give an adaptation of the proof in \cite{CHZ24} in order to match $\@I_0$, which is slightly different from the instances in \cite{CHZ24}. The intuition of the proof is that if the algorithm retains $\arm_i$ with a higher probability in $\rho_i$ than in $\rho_1$, then this algorithm should pull $\arm_i$ enough times; otherwise, it cannot distinguish these two instances well. 

Let $\tau_i$ denote the event that ``$\+A$ retains $\arm_i$''. Then, for any algorithm $\+A$, $m = \E[\rho_1]{\sum_{i=1}^k \1{\tau_i}} = \sum_{i=1}^k \Pr[\rho_1]{\tau_i}$. Thus we have $\sum_{i:2\leq i\leq k}\frac{\Pr[\rho_1]{\tau_i}}{k-1} \leq \frac{m-(1-\delta)}{k-1} \leq 1-\delta$ due to the fact that $\+A$ is $(\eps,\delta)$-PAC and $\delta \leq \frac{k-m}{k}$. 

Therefore for any $2\leq i\leq k$, apply \Cref{lem:KL-pull-lb} to $\rho_1$ and $\rho_i$ with $\tau_i$,
\begin{equation} \label{eqn:i-lb}
    \E[\rho_1]{T_i}\cdot d\tp{\frac{1}{2} + (n-1)\eps, \frac{1}{2} + (n+1)\eps} \ge d\tp{ \Pr[\rho_1]{\tau_i}, \Pr[\rho_i]{\tau_i} }.
\end{equation}

Since $\frac{1}{2} \leq \frac{1}{2}+(n-1)\eps \leq \frac{5}{8}$ and $2\eps \leq \frac{1}{8}$, we apply \Cref{pro:KL-taylor} and obtain that \[d\tp{\frac{1}{2} + (n-1)\eps, \frac{1}{2} + (n+1)\eps} \allowbreak \leq 16\eps^2.\]
Now we sum up all $i$ of \cref{eqn:i-lb},
\begin{align*}
    \E[\rho_1]{T} 
    &\ge \frac{1}{16\eps^2} \sum_{2\leq i\leq k} d\tp{\Pr[\rho_1]{\tau_i}, \Pr[\rho_i]{\tau_i}} \\
    \mr{\Cref{pro:KL-ineq}}
    &\ge \frac{1}{16\eps^2} \sum_{2\leq i\leq k} d\tp{\Pr[\rho_1]{\tau_i}, 1-\delta} \\ 
    \mr{\Cref{lem:sum-ineq}}
    &\ge \frac{(k-1)}{16\eps^2}\cdot d\tp{\frac{m-(1-\delta)}{k-1}, 1-\delta}, 
\end{align*}
Let $\delta = \frac{k-m}{k}(1-\alpha)$, while $\beta \leq \alpha < 1$. And $r \defeq \frac{1 + \frac{\alpha}{k-1} - (1-\alpha) }{1-\alpha} = \frac{k\alpha}{(k-1)(1-\alpha)} > \frac{\alpha}{1-\alpha}$, then 
\begin{align*}
    d\tp{\frac{m-(1-\delta)}{k-1}, 1-\delta} 
    &= d\tp{1 - \frac{m-(1-\delta)}{k-1}, \delta} = d\tp{\frac{k-m-\delta}{k-1}, \delta} \\
    &= d\tp{\frac{k-m}{k}(1 + \frac{\alpha}{k-1}), \frac{k-m}{k}(1-\alpha)} \\
    \mr{\Cref{lem:KL-first-item-ineq}}
    &\ge \frac{r}{2+2r} \cdot \frac{k-m-\delta}{k-1} \log  \frac{k-m-\delta}{(k-1)\delta} \\
    &\ge \frac{\frac{\alpha}{1-\alpha}}{2+\frac{2\alpha}{1-\alpha}}  \cdot \frac{k-m-\delta}{k-1} \log  \frac{k-m-\delta}{(k-1)\delta} \\
    &\ge \frac{\beta}{2}\cdot \frac{k-m-\delta}{k-1} \log  \frac{k-m-\delta}{(k-1)\delta},
\end{align*}
Bring this item back and we could get
\[
    \E[\rho_1]{T} \ge \frac{\beta}{32}\cdot \frac{k-m-\delta}{\eps^2}\log{\frac{k-m-\delta}{(k-1)\delta}}. 
\]

\end{document}